\documentclass[sn-mathphys-num]{sn-jnl}
\usepackage{graphicx}%
\usepackage{multirow}%
\usepackage{natbib}
\usepackage{amsmath,amssymb,amsfonts}%
\usepackage{amsthm}%
\usepackage{mathrsfs}%
\usepackage[title]{appendix}%
\usepackage{xcolor}%
\usepackage{textcomp}%
\usepackage{manyfoot}%
\usepackage{booktabs}%
\usepackage{algorithm}%
\usepackage{algorithmicx}%
\usepackage{algpseudocode}%
\usepackage{listings}%
\theoremstyle{thmstyleone}%
\newtheorem{theorem}{Theorem}
\newtheorem{proposition}[theorem]{Proposition}%
\usepackage{pdfpages}
\usepackage{subcaption}

\usepackage{threeparttable}

\theoremstyle{thmstyletwo}%

\theoremstyle{thmstylethree}%
\newtheorem{definition}{Definition}%

\newcommand{\TheName}[0]{\textbf{Knoop}}
\newcommand{\Kappa}[0]{\mathbf{K}}

\begin{document}

\title[\TheName: Practical Enhancement of Knockoff with Over-Parameterization for Variable Selection]{\TheName: Practical Enhancement of Knockoff with Over-Parameterization for Variable Selection\vspace{-5mm}}

\author[1]{\fnm{Xiaochen} \sur{Zhang}}\email{202012079@mail.sdu.edu.cn}
\author[2]{\fnm{Yunfeng} \sur{Cai}}\email{yunfengcai09@gmail.com\vspace{-5mm}}
\author*[3]{\fnm{Haoyi} \sur{Xiong}}\email{haoyi.xiong.fr@ieee.org}


\affil[1]{\orgdiv{Research Center for Mathematics and Interdisciplinary Sciences}, \orgname{Shandong University}, \postcode{266237}, \orgaddress{\city{Qingdao}, \country{China}}}

\affil[2]{\orgname{Yanqi Lake Beijing Institute of Mathematical Sciences And Applications}, \orgaddress{ \city{Huairou District}, \postcode{100084}, \state{Beijing}, \country{China}}}

\affil[3]{\orgname{Microsoft}, \city{Haidian District}, \postcode{100085}, \state{Beijing}, \country{China}\vspace{-8mm}}

\abstract{
Variable selection plays a crucial role in enhancing modeling effectiveness across diverse fields, addressing the challenges posed by high-dimensional datasets of correlated variables. This work introduces a novel approach namely \emph{\underline{Kn}ockoff with \underline{o}ver-\underline{p}arameterization} (\TheName{}) to enhance Knockoff filters for variable selection. Specifically, \TheName{} first generates multiple knockoff variables for each original variable and integrates them with the original variables into an over-parameterized Ridgeless regression model. For each original variable, \TheName{} evaluates the coefficient distribution of its knockoffs and compares these with the original coefficients to conduct an anomaly-based significance test, ensuring robust variable selection. Extensive experiments demonstrate superior performance compared to existing methods in both simulation and real-world datasets. \TheName{} achieves a notably higher Area under the Curve (AUC) of the Receiver Operating Characteristic (ROC) Curve for effectively identifying relevant variables against the ground truth by controlled simulations, while showcasing enhanced predictive accuracy across diverse regression and classification tasks. The analytical results further backup our observations.

\vspace{-5mm}}
\maketitle

\vspace{-1mm}\section{Introduction\vspace{-1mm}}
Variable selection, especially supervised variable/feature selection, plays a crucial role in enhancing modeling effectiveness across diverse fields, such as genetics and biological science \cite{he2022ghostknockoff}, especially when dealing with high-dimensional data characterized by complex correlations and heterogeneity. By forming a parsimonious subset of key variables, this process addresses challenges like the curse of dimensionality and the presence of irrelevant and redundant variables. It thus improves the stability of regression models and the clarity of predictions. Two categories of variable selection strategies, model-based and model-agnostic, are distinguished by their underlying assumptions and selection approaches.
Common variable selection methods include multiple testing and sparse (linear) models, such as least absolute shrinkage and selection operator (Lasso)~\cite{Lasso1, frandi2016fast}, ElasticNet~\cite{elasticNet}, and sparse autoencoder~\cite{atashgahi2022quick}. Multiple testing evaluates variable relevance through investigating statistically conditional independence among variables~\cite{multipleTesting,watson2021testing}, while Lasso applies an $\ell_1$-norm penalty to reduce non-zero coefficients, thereby stabilizing the model and aiding selection by compressing coefficients. In addition, the Dantzig selector~\cite{candes2007dantzig}, named after George Dantzig, operates by minimizing the residual sum of squares subject to an $\ell_1$-norm constraint on the coefficients, which offers advantages in high-dimensional settings. On the other hand, ElasticNet~\cite{elasticNet} combines $\ell_1$-norm and $\ell_2$-norm penalties to overcome the drawbacks of Lasso by introducing a convex combination of both penalties. This hybrid approach retains the variable selection properties of Lasso while encouraging grouping effects and handling multicollinearity more effectively. However, multiple testing can oversimplify complex datasets, and sparse linear models might introduce bias in the presence of highly correlated variables. These limitations motivate our research on variable selection techniques.

More recently, Knockoff has been proposed as a statistical variable selector aiming at controlling the False Discovery Rate (FDR) in variable selection \cite{fixedKnockoff, modelXKnockoff}. This approach involves generating synthetic ``knockoff'' versions of variables, which are then compared with the original variables to identify significant differences. Moreover, Knockoff demonstrates robustness in handling highly correlated variables, a common challenge for sparse linear models. The capability of Knockoff to create synthetic references for comparison helps disentangle the effects of correlated variables more effectively, leading to more stable and unbiased variable selection outcomes~\cite{watson2021testing,yu2020causality}. In addition to its FDR control and robustness, the Knockoff method stands out for its versatility in handling complex models without imposing restrictive assumptions. Advanced Knockoff filters, such as Model-X~\cite{modelXKnockoff}, are designed to work with diverse models, expanding its applicability across various fields and datasets. Such adaptability allows Knockoff to outperform sparse linear models in scenarios where traditional techniques may struggle due to data complexity or model assumptions. Notably, Knockoff has demonstrated superior performance in genomics and high-dimensional data settings~\cite{he2022ghostknockoff}, showcasing its efficacy in accurately identifying relevant variables amidst a large number of potentially correlated variables.

\begin{figure}
\vspace{-3mm}
    \centering
    \subfloat[Knockoff with Ridge Regression]{\includegraphics[width=0.7\textwidth]{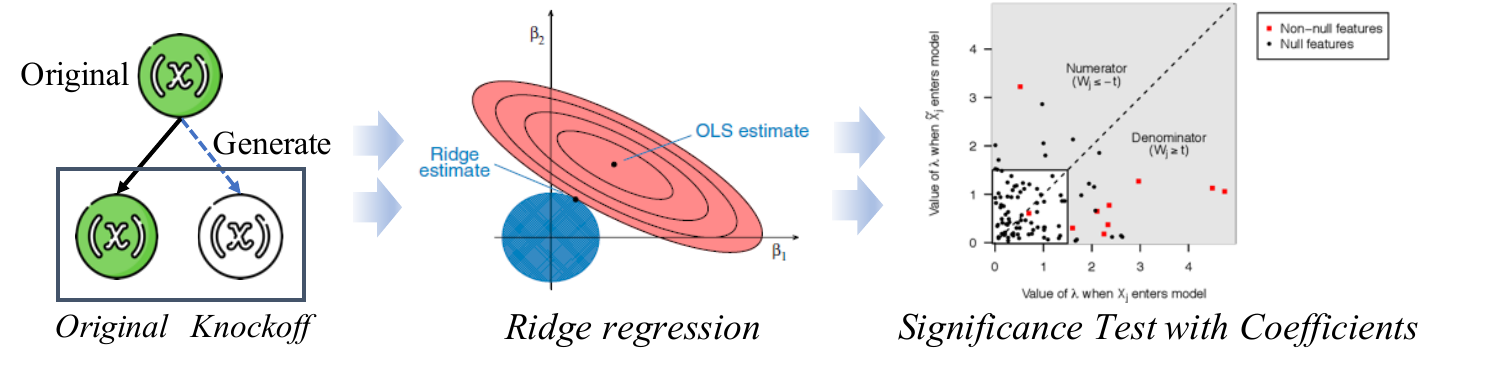}}\vspace{-3mm}\\
    \subfloat[The Proposed Pipeline of \TheName{}]{\includegraphics[width=\textwidth]{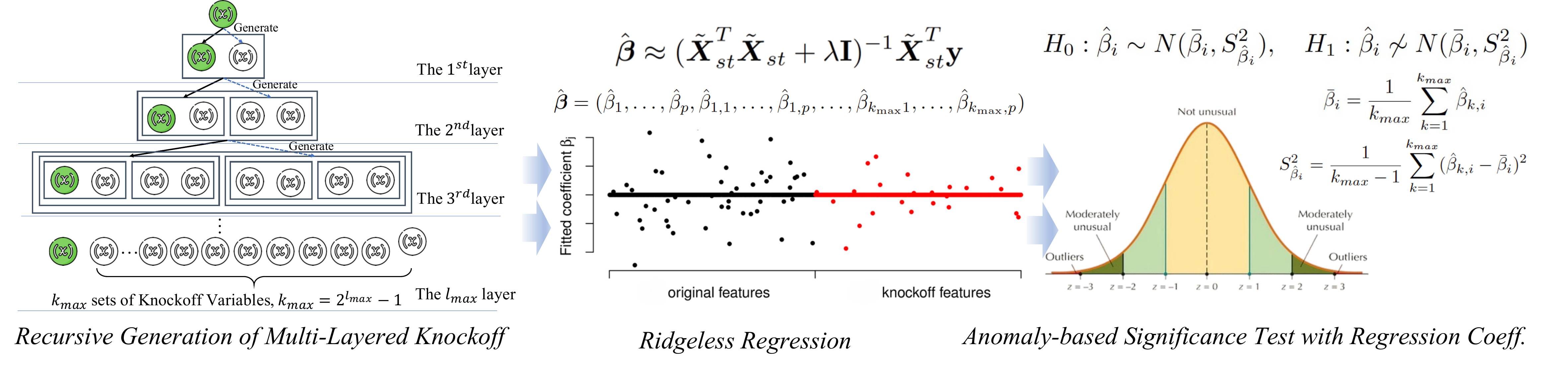}}\vspace{-2mm}
    \caption{A Brief Comparison between Knockoff and the proposed \TheName{} pipeline}
    \label{fig:pipeline-compare}\vspace{-6mm}
\end{figure}

While Knockoff could outperform sparse linear models in handling highly correlated variables, it still suffers performance limitations from the two major perspectives: computational costs and model capacity. Though the computational costs associated with generating knockoff variables might be substantial, this challenge can be mitigated by leveraging the rapid advancements in computational power with advanced computing~\cite{ko2022high}. In terms of model capacity, the Knockoff filter only generates a set of ``knockoff'' variables from the original ones, and fits the data using the two set of variables via linear regression model like Ridge~\cite{hoerl1970Ridge}. Such simple model might be insufficient to capture complex relationships between variables and the response variable. In addition, the regularization made by Ridge would also induce bias into the model estimation. Thus, it is reasonable to doubt whether the performance of Knockoff could be further improved when the regression model fits the data ``perfectly''~\cite{liang2020just,RidgelessTheory5}.

To address above issues, we introduce a novel Knockoff method, namely \emph{Knockoff with Over-Parameterization} (\TheName{}),  that works on top of over-parameterized linear models (Fig.~\ref{fig:pipeline-compare} presents a simple comparison on the pipeline). The proposed method first leverages a \emph{recursive approach to generating multi-layered knockoff variables} from the original ones, where the knockoff variables are expected to be independent of the response variable yet exchangeable with the original ones. Then, \TheName{} incorporates both the original variables and their knockoff copies into a Ridgeless regression model~\cite{liang2020just} subject to the response variable, refining the coefficient estimates with enhanced fitness. Further, based on the Ridgeless estimates of coefficients for all variables, \TheName{} proposes a procedure of \emph{anomaly-based significant test} to compare the coefficients of original variable against knockoff counterparts for FDR control and rank the original variables by their significance for variable selection. Extensive experiments have been done to demonstrate superior performance compared to existing methods in both simulation and real-world datasets, for either relevant variable identification and feature selection for predictive modeling tasks.

\vspace{-1mm}\section{Backgrounds and Preliminaries}\vspace{-1mm}

\vspace{-1mm}\subsection{Variable Selection via Linear Models\vspace{-1mm}}
Consider a dataset with \(n\) independent and identically distributed observations of $p$ variables, which can be collectively represented in matrix form as \({X} \in \mathbb{R}^{n \times p}\) for the variable matrix and \(\mathbf{y} \in \mathbb{R}^n\) for the response vector. An unknown projection vector $\boldsymbol{\beta}=(\beta_1, \ldots, \beta_p)\in \mathbb{R}^p$ is assumed to represent a linear relationship as follow
\[
\mathbf{y} = {X} \boldsymbol{\beta} + \boldsymbol{\varepsilon}\ \text{and}\  \boldsymbol{\varepsilon}\sim \mathcal{N}(0, \sigma^2 I)\ .
\]
The objective of variable selection is to infer the non-zero coefficients in $\boldsymbol{\beta}$, so as to understand the relationship between the variables and their responses.

\vspace{-1mm}
\subsection{Knockoff Methods}\vspace{-1mm}
Generally, two categories of Knockoffs --- fixed-X and model-X methods have been studied~\cite{fixedKnockoff,modelXKnockoff}. While the fixed-X Knockoff filter generates synthetic variables with marginal correlations, independent of other variables, our work focuses on the model-X Knockoff filter generating knockoffs by modeling the joint distribution of all variables, thus managing more complex dependencies.  
The model-X knockoff variables \(\widetilde{X}\) replicate the distribution of the original variables \(X\)~\cite{modelXKnockoff}. These variables meet two crucial criteria: \emph{Exchangeability}: \((X, \widetilde{X})_{\mathrm{swap}(S)} \overset{d}{=} (X, \widetilde{X})\), achieved by switching each \(X_j\) with \(\widetilde{X}_j\) in any subset \(S\) of indices and \emph{Independence}: Ensuring that \(\widetilde{X}\) is independent of the response vector \(\mathbf{y}\), conditional on \(X\). In the model-X Knockoff filter, the importance of each variable \(X_j\) is assessed through a statistic \(W_j = w_j([X, \widetilde{X}], \mathbf{y})\), where \(w_j\) is a function indicating the relevancy of variables. Importantly, the sign of \(w_j\) is reversed when \(X_j\) and \(\widetilde{X}_j\) are swapped within \(S\).  
\begin{align}
w_j([X, \widetilde{X}]_{\mathrm{swap}(S)}, \mathbf{y}) = 
\begin{cases}
w_j([X, \widetilde{X}], \mathbf{y}), & j \notin S, \\
-w_j([X, \widetilde{X}], \mathbf{y}), & j \in S.
\end{cases}
\end{align}
A threshold \(\tau\) is established based on the condition as follows.
\begin{align}
\tau = \min\left\{ t > 0: \frac{1 + \#\{j : W_j \leq -t\}}{\#\{j : W_j \geq t\}} \leq q \right\},
\end{align}
which guides the selection of variables for which \(W_j \geq \tau\). By adhering to the above criterion, the method ensures that the rate of false discoveries among the selected variables remains within the pre-established FDR level \(q\). To further enhance the stability of model-X Knockoff, the multiple Knockoff filter~\cite{nguyen2020aggregation} has been proposed to run multiple instances of the Knockoff procedure simultaneously and aggregating their results to control the FDR more effectively for variable selection. 

Unlike existing methods, our approach \TheName{} generates multiple sets of knockoff variables to approximate a distribution of coefficients for the knockoff, and uses both original variables and multiple knockoff sets into the regression model, enhancing both scalability and capability. Finally an anomaly-based test is proposed by \TheName{} to compare the coefficient of every single original variable with the distribution of its knockoff counterparts for significance evaluation. Such inclusion creates a larger control group of null variables, maintaining distributional similarity but ensuring independence from the response variable, thus providing a robust baseline for assessing the impact of original variables.

\vspace{-1mm}
\subsection{Ridge Regression and Sparse Linear Models}\vspace{-1mm}
When selecting variables from data, one approach is to assess the significance of the coefficients in regression models,
as follows.  
\begin{itemize}
    \item \emph{Ridge Regression~\cite{hoerl1970Ridge}:} Ridge Regression integrates an \(\ell_2\) penalty into the loss function, effectively shrinking regression coefficients towards zero. This regularization helps mitigate overfitting by penalizing the magnitude of the coefficients. The solution is formulated as:
    \begin{align}
    \hat{\boldsymbol{\beta}}_{r} = \underset{\beta \in \mathbb{R}^{p}}{\text{argmin}} \left\{ \frac{1}{2} \| \mathbf{y} - X\beta \|_2^2 + \lambda_{r} \| \beta \|_2^2 \right\},
    \end{align}
    where \(\lambda_{r}\) is the regularization strength.
    
    \item \emph{Lasso Regression~\cite{Lasso1}:} Lasso promotes sparsity in the coefficient vector through an \(\ell_1\) penalty, setting less important variable coefficients to zero, thus aiding in variable selection. The coefficients are optimized by:
    \begin{align}
    \hat{\boldsymbol{\beta}}_{l} = \underset{\beta \in \mathbb{R}^{p}}{\text{argmin}} \left\{ \frac{1}{2} \| \mathbf{y} - X\beta \|_2^2 + \lambda_{l} \| \beta \|_1 \right\},
    \end{align}
    balancing the model fit and complexity via the tuning parameter \(\lambda_{l}\).
    
    \item \emph{ElasticNet~\cite{elasticNet}:} Elastic Net merges the penalties of Lasso and Ridge Regression to harness the benefits of both sparsity and stability.
\end{itemize}

\vspace{-1mm}
\subsection{Ridgeless Regression and Overparameterization}\vspace{-1mm}
Ridgeless regression extends the ordinary least squares estimator to provide a unique solution in ultra-high dimensional settings. The coefficients, $\hat{\boldsymbol{\beta}}$, can be obtained through a limiting process as $\lambda \to 0$:
\begin{align}
    \hat{\boldsymbol{\beta}} &= \lim_{\lambda \to 0} (X^T X + \lambda I)^{-1} X^T \mathbf{y},
\end{align}
Theoretical research has explored the conditions under which over-parameterized Ridgeless regression models attain advantages of estimation and prediction~\cite{equivalentRidgeless}.

\vspace{-1mm}
\section{Methodology: Over-Parameterized Knockoff}
\label{sec:section3}\vspace{-1mm}
This section presents the overall framework and core algorithms of \TheName{}.

\vspace{-1mm}
\subsection{Overall Framework Design}\vspace{-1mm}
In this section, we describe the comprehensive design of the \TheName{} algorithm (Algorithm~\ref{alg:Knoop}). The algorithm accepts the training set \((X, \mathbf{y})\) as input and outputs the selection result of the variables. The method comprises four principal steps: \emph{hierarchical generation of multi-layered knockoffs, Ridgeless regression, anomaly-based significance test, and variable selection.} Each step is elaborated upon as follows.

\paragraph{Recursive Generation of Multi-Layered Knockoffs}
Given the training set \((X, \mathbf{y})\) as input, \TheName{} employs a unique recursive generation approach to generate enhanced knockoff variable matrix \(X\), outlined in lines 1 of Algorithm~\ref{alg:Knoop}. This approach generates a multi-layered knockoff matrix, \(\widetilde{X}\), through a recursive process, thereby increasing the model capacity with a larger number of dimensions in input features for regression. The resulting \(\widetilde{X}\) is represented as:
\[
\widetilde{X} = [X_{1}, \ldots, X_{p}, \widetilde{\Kappa}_{1,1}, \ldots, \widetilde{\Kappa}_{1,p}, \ldots, \widetilde{\Kappa}_{k_{\text{max}},1}, \ldots, \widetilde{\Kappa}_{k_{\text{max}},p}]\ .
\]
Each \(\widetilde{\Kappa}_{i}=[\widetilde{\Kappa}_{i,1}, \widetilde{\Kappa}_{i,2}\dots,\widetilde{\Kappa}_{i,p}]\) represents a hierarchical knockoff matrix associated with \(X\), where \(i\) ranges from \(1\) to \(k_{\text{max}}\). Later, to mitigate the influence of varying variable units, each column of the matrix is normalized to unit norm, resulting in a standardized multi-knockoff matrix, denoted as \(\widetilde{X}_{\text{st}}\).

\vspace{-3mm}
\begin{algorithm}
\caption{\texttt{Knoop}: Knockoff with Over-Parameterization}
\label{alg:Knoop}
\begin{algorithmic}[1]
\Statex \textbf{Input data:} Training data matrix \( X = [x_1, \ldots, x_{n}]^T \), Labels for training data \( \mathbf{y} \).
\Statex \textbf{Parameters:} Ridgeless regression regularization parameter \( \lambda \).
\Statex \textbf{Output:} 
$p$-values vector \(\hat{\boldsymbol{P}}\) for original variables \(x_1, \ldots, x_p\).
Set of selected indices \( A \) within \( \{1, 2, \ldots, p\} \).

\State \(\widetilde{X} \gets \texttt{recursKnockoff}(X)\) \Comment{Generate a multi-layered knockoff matrix}
\State \( \widetilde{X}_{st}\gets\texttt{normalize}(\widetilde{X}) \) \Comment{Normalize each column to unit norm}

\State $\hat{\boldsymbol{\beta}} \gets (\widetilde{X}_{st}^\top \widetilde{X}_{st} + \lambda I)^{-1} \widetilde{X}_{st}^\top \mathbf{y}$ \Comment{Estimate the coefficients via Ridgeless regression}

\State \(\hat{\boldsymbol{P}} \gets \texttt{estPVal}(\hat{\boldsymbol{\beta}},\widetilde{X}_{st}, \mathbf{y})\) \Comment{Estimate $p$-values by significance tests}

\State \(A \gets \texttt{selectVars}(\hat{\boldsymbol{P}})\) \Comment{Select top variables by $p$-values}

\State \Return \( \hat{\boldsymbol{P}} \), \( A \), \Comment{Return $p$-values and selected variables}.
\end{algorithmic}
\end{algorithm}
\vspace{-5mm}

\paragraph{Ridgeless Regression}
Following the above step, \TheName{} integrates both the original variables and multiple knockoff vectors into a linear regression model, utilizing a Ridgeless least squares estimator to derive the coefficient estimates, as detailed in Lines 3 of Algorithm~\ref{alg:Knoop}. \TheName{} obtains a vector of estimated coefficients, \(\hat{\boldsymbol{\beta}}\), as:
\begin{align}
\hat{\boldsymbol{\beta}} = (\hat{\beta}_1, \ldots, \hat{\beta}_p, \hat{\beta}_{1,1}, \ldots, \hat{\beta}_{1,p}, \ldots, \hat{\beta}_{k_{\text{max}},1}, \ldots, \hat{\beta}_{k_{\text{max}},p}),
\end{align}
where \(\hat{\beta}_1, \ldots, \hat{\beta}_p\) are the coefficients of the original variables. Each \(\hat{\beta}_{k,i}\) are the coefficient for the \(k^{th}\) knockoff of the \(i^{th}\) variable, for \(k = 1, \ldots, k_{\text{max}}\) and \(i = 1, \ldots, p\).

\paragraph{Anomaly-based Significance Test}
Given the Ridgeless regression coefficients, \TheName{} estimates the $p$-values for the significance tests of the original variables, as executed in Lines 4 of Algorithm~\ref{alg:Knoop}. This process involves fitting a Ridgeless regression model to the expanded training set \((\widetilde{X}_{\text{st}}, \mathbf{y})\). For a sufficiently large \(k_{\text{max}}\) (ensuring that \((k_{\text{max}}+1)p \geq n_{train}\)), the regression model perfectly fits the training data. The estimated coefficients, denoted \(\hat{\boldsymbol{\beta}}\), are then used to derive the $p$-values \(\hat{P}_i\) for assessing variable significance. The resulting $p$-values \(\hat{P}_i\), where smaller values indicate higher variable importance, are used to rank the variables. This step outputs a vector \(\hat{\boldsymbol{P}}\), containing calculated $p$-values for each original variable.

\paragraph{Variable Selection}\label{sec:var_select}
As was mentioned, the previous step of \TheName{} offers a significance ranking of variables based on estimated $p$-values. To implement \(\texttt{selectVars}(\cdot)\) in Line 5 of Algorithm~\ref{alg:Knoop}, \TheName{} provides two ways for variable selection as follows.
\begin{itemize}    
    \item \emph{Fixed-length Selection}: In addition to variable selection by adjusted $p$-values (such as the Benjamini-Hochberg procedure~\cite{benjamini1995controlling}), yet another simple selection method here involves identifying the desired number of top variables with the lowest $p$-values. This method is particularly useful when the analysis has a pre-defined budget for the number of variables. However, when no prior information is available, determining the appropriate number of variables may necessitate model selection.

    \item \emph{Varying-length Selection}: As was mentioned above, in scenarios where the optimal number of variables for selection varies, validation-based model selection becomes essential. This approach utilizes an additional validation dataset or cross-validation~\cite{energy} on the training data to tune-and-error the model and determine the appropriate number of variables.

\end{itemize}
%
%
In following sections, we present the details of two core algorithms --- \emph{recursive generation of multi-layered knockoffs} and \emph{anomaly-based significance tests}.

\vspace{-1mm}
\subsection{Recursive Generation of Multi-Layered Knockoffs}\vspace{-1mm}
This section outlines the proposed algorithm for recursive generation of multi-layered knockoffs, detailed in Algorithm~\ref{alg:hierKnock}. Given a data matrix based on the $n$ samples of the original $p$ variables, denoted as $X$ --- a $n\times p$ matrix, and the number of layers desired $\ell_{max}$, this step would recursively generate $2^{\ell_{max}}-1$ sets of knockoff variables, totalling $(2^{\ell_{max}}-1)p$ variables, from the original variables. It returns a $n\times ( 2^{\ell_{max}}p)$ knockoff matrix, denoted as $\widetilde{X}$, including both original variables and generated ones. Note that to simplify our analysis in the rest of our paper, we use $k_{max}=2^{\ell_{max}}-1$.

\vspace{-2mm}
\begin{algorithm}
\caption{\texttt{recursKnockoff}: Recursive Generation of Multi-Layered Knockoffs}
\label{alg:hierKnock}
\begin{algorithmic}[1]
\Statex \textbf{Input:} Original variable matrix \(X\), the number of layers desired \(\ell_{max}\)\ (note that we use \(k_{max}=2^{\ell_{max}}-1\) to simplify our analysis in the rest part of this article), regularization parameter \( \lambda \).
\Statex \textbf{Output:} Multi-layered knockoff matrix \(\widetilde{X}\).

\State \(\kappa^1 \gets \text{computeKnockoff}(X)\) \Comment{Compute the initial knockoff matrix}
\State \(\Kappa^1 = [X \,|\, \kappa^1]\) \Comment{Concatenate \(X \text{ and } \kappa^\ell\) into a new matrix $\Kappa^1$}

\For{\(\ell = 1\) to \(\ell_{max}-1\)}
    \State \(\kappa^{\ell+1} \gets \texttt{computeKnockoff}(\Kappa^{\ell})\) \Comment{Compute the $\ell^{th}$-layer knockoff matrix}
    \State \(\Kappa^{\ell+1} \gets [\Kappa^\ell \,|\, \kappa^{\ell+1}]\) \Comment{Concatenate \(\Kappa^\ell \text{ and } \kappa^{\ell+1}\) into a new matrix $\Kappa^{\ell+1}$}
\EndFor
\State \Return  \(\widetilde{X} = \Kappa^{\ell_{max}}\) \Comment{Return a multi-layered ($\ell_{max}$ layers) knockoff matrix}
\end{algorithmic}
\end{algorithm}
\vspace{-3mm}

\subsubsection{Algorithm Design}
Initially, the matrix \(\kappa^1\), denoted as the matrix of knockoff counterparts of the original variable matrix \(X\), is computed using $\texttt{computeKnockoff}(\cdot)$ and then horizontally concatenated with \(X\) to form the first layer \(\Kappa^1\) (lines 1-2 in Algorithm~\ref{alg:hierKnock}):
\[
\Kappa^1 = [X \,|\, \kappa^1]
\]
Subsequent layers are constructed by repeating this process for each layer, where a new knockoff \(\kappa^{\ell+1}\) is created for each existing layer \(\Kappa^\ell\) via $\texttt{computeKnockoff}(\cdot)$ and concatenated to form \(\Kappa^{\ell+1}\) (lines 3-6 in Algorithm~\ref{alg:hierKnock}):
\[
\Kappa^{\ell+1} = [\Kappa^\ell \,|\, \kappa^{\ell+1}]
\]
The procedure continues until the desired number of layers \(\ell_{max}\) is reached, culminating in the final hierarchical knockoff matrix \(\widetilde{X} = \Kappa^{\ell_{max}}\). Note that, in the rest of this paper, we use a layer-column conversion and write the structure of $\widetilde{X}$ as follows.
\[
\widetilde{X} = [X_{1}, \ldots, X_{p}, \widetilde{\Kappa}_{1,1}, \ldots, \widetilde{\Kappa}_{1,p}, \ldots, \widetilde{\Kappa}_{k_{\text{max}},1}, \ldots, \widetilde{\Kappa}_{k_{\text{max}},p}]\ ,
\]
where $[\widetilde{\Kappa}_{2^{\ell-1},1},\dots,\widetilde{\Kappa}_{2^{\ell-1},p},\dots, \widetilde{\Kappa}_{2^{\ell}-1,1},\dots,\widetilde{\Kappa}_{2^{\ell}-1,p}]=\kappa^\ell$ for each $\ell=1,\ 2\dots, \ell_{max}$. For example, $[\widetilde{\Kappa}_{1,1},\dots,\widetilde{\Kappa}_{1,p}]=\kappa^1$ and $[\widetilde{\Kappa}_{2,1},\dots,\widetilde{\Kappa}_{2,p},\widetilde{\Kappa}_{3,1},\dots,\widetilde{\Kappa}_{3,p}]=\kappa^2$.  
The overall generation $\widetilde{X}$ totals $p$ original variables and $(2^{\ell_{max}}-1)\cdot p$ knockoff variables (as $k_{max}=2^{\ell_{max}}-1$), which is then utilized for further analysis. 

\subsubsection{Algorithm Analysis}
While traditional Knockoff methods focus on ensuring exchangeability between each original variable and its knockoff, \TheName{} is expected to maintain such exchangeability for multiple sets of knockoffs generated recursively from the same original variable. Here, we provide a brief analysis on the exchangeability of generated variables to the original ones~\cite{knockoffGeneration1,knockoffGeneration2}. To establish our analysis, we first make the following assumptions.
\begin{itemize}
    \item[\em A1] We adopt the knockoff generation procedure described in Model-X Knockoff filters~\cite{modelXKnockoff}. Specifically, for any input matrix $X_{inp}$ and its generated knockoff matrix $X_{gen}=\texttt{computeKnockoff}(X_{inp})$, in addition to the exchangeability between each column in  $X_{inp}$ and its counterpart in $X_{gen}$, we assume that each column in $X_{inp}$ and $X_{gen}$ is distinct.

   \item[\em A2] We assume the \emph{transitivity} of exchangeability among multiple random matrices. Let denote $V_1$, $V_2$, and $V_3$ as three independent random matrices. Suppose $V_1$ and $V_2$ exchangeable and  $V_2$ and  $V_3$ are exchangeable. We say $V_1$ and $V_3$ should be exchangeable.
\end{itemize}
The assumption is critical to ensure the \emph{recursive generation of multi-layered knockoff} approach would not create knockoff variables that are identical to its original variables or other knockoff ones. In our experiments, we find this assumption holds in both simulation studies or evaluations based on real-world datasets. Based on the above assumptions, we derive the main analysis result for the \emph{recursive generation of multi-layered knockoffs} as following proposition.

\begin{proposition}[Exchangability between multi-layered knockoffs]\label{prop:exchange} Given the input matrix $X=[X_{1}, \ldots, X_{p}]$, Algorithm~\ref{alg:hierKnock} outputs a matrix $\widetilde{X}$ with the structure $\widetilde{X}=[X_{1}, \ldots, X_{p}, \widetilde{\Kappa}_{1,1}, \ldots, \widetilde{\Kappa}_{1,p}, \ldots, \widetilde{\Kappa}_{k_{\text{max}},1}, \ldots, \widetilde{\Kappa}_{k_{\text{max}},p}]$. We say that for each $j=1,\ 2,\dots, k_{max}$, each $j'=1,\ 2,\dots, k_{max}$ and $j\neq j'$, the matrix $X$ and the sub-matrix $\widetilde{\Kappa}_{j} = [\widetilde{\Kappa}_{j,1}\dots\widetilde{\Kappa}_{j,p}]$ of $\widetilde{X}$ are exchangeable, such that
\begin{equation}\label{eq:X_Kappa_swap}
[X_{1},\dots,X_p,\widetilde{\Kappa}_{j,1}\dots\widetilde{\Kappa}_{j,p}]_{\mathrm{swap}(S)} \overset{d}{=}[X_{1},\dots,X_p,\widetilde{\Kappa}_{j,1}\dots\widetilde{\Kappa}_{j,p}],
\end{equation}
and the sub-matrices $[\widetilde{\Kappa}_{j,1},\dots,\widetilde{\Kappa}_{j,p}]$ and $[\widetilde{\Kappa}_{j',1}\dots,\widetilde{\Kappa}_{j',p}]$ are exchangeable, such that
\begin{align}
\label{eq:Kappa_Kappa_prime_swap}
[\widetilde{\Kappa}_{j,1}\dots\widetilde{\Kappa}_{j,p},\widetilde{\Kappa}_{j',1}\dots\widetilde{\Kappa}_{j',p}]_{\mathrm{swap}(S)} \overset{d}{=}[\widetilde{\Kappa}_{j,1}\dots\widetilde{\Kappa}_{j,p},\widetilde{\Kappa}_{j',1}\dots\widetilde{\Kappa}_{j',p}],
\end{align}
where the operator $A_{\mathrm{swap}(S)} \overset{d}{=}A$ for a matrix $A$ was defined in the \textbf{Definition 2} in~\cite{modelXKnockoff} representing the distributional invariance of the matrix after certain column-wise swapping --- a cornerstone of exchangeability.
\end{proposition}

\begin{proof}
We provide a brief proof based on following steps:

\begin{itemize}
    \item[] \textbf{Step 1.} By the exchangeability property in the definition of model-$X$ knockoff given by Definition 2 in \cite{modelXKnockoff}, for the input matrix $X=[X_{1}, \ldots, X_{p}]$ and its knockoff matrix $\kappa^1 = [\widetilde{\Kappa}_{1,1},\dots,\widetilde{\Kappa}_{1,p}] = \texttt{computeKnockoff}(X)$, we have:
    $[X_{1},\dots,X_p,\widetilde{\Kappa}_{1,1}\dots\widetilde{\Kappa}_{1,p}]_{\mathrm{swap}(S)} \overset{d}{=}[X_{1},\dots,X_p,\widetilde{\Kappa}_{1,1}\dots\widetilde{\Kappa}_{1,p}].$

    \item[] \textbf{Step 2:} Then we prove that the second layer of $\widetilde{X}$ matrix, \(\widetilde{\Kappa}^2\), satisfies Proposition~\ref{prop:exchange}. By definition of knockoff, 
    \begin{align}
    \label{eq:step2_knockoff_equality}
    \kappa^2 \overset{d}{=} \widetilde{\Kappa}^1
    \end{align}
    Therefore, $X \overset{d}{=} [\widetilde{\Kappa}_{2,1}, \dots, \widetilde{\Kappa}_{2,p}]$, and $[\widetilde{\Kappa}_{1,1}, \dots, \widetilde{\Kappa}_{1,p}] \overset{d}{=} [\widetilde{\Kappa}_{3,1}, \dots, \widetilde{\Kappa}_{3,p}]$ so that $X$ and $[\widetilde{\Kappa}_{2,1}, \dots, \widetilde{\Kappa}_{2,p}]$ satisfies Equation~\ref{eq:X_Kappa_swap}, $[\widetilde{\Kappa}_{1,1}, \dots, \widetilde{\Kappa}_{1,p}] \text{ and } [\widetilde{\Kappa}_{3,1}, \dots, \widetilde{\Kappa}_{3,p}]$ satisfies Equation~\ref{eq:Kappa_Kappa_prime_swap}. Besides, $[\widetilde{\Kappa}_{2,1}, \dots, \widetilde{\Kappa}_{2,p}]$ and $[\widetilde{\Kappa}_{3,1}, \dots, \widetilde{\Kappa}_{3,p}]$ satisfies Equation~\ref{eq:Kappa_Kappa_prime_swap}.

    So by Lemma \emph{A2}, $X$ and $[\widetilde{\Kappa}_{3,1}, \dots, \widetilde{\Kappa}_{3,p}]$ satisfies Equation~\ref{eq:X_Kappa_swap} and $[\widetilde{\Kappa}_{1,1}, \dots, \widetilde{\Kappa}_{1,p}]$ and $[\widetilde{\Kappa}_{2,1}, \dots, \widetilde{\Kappa}_{2,p}]$ satisfies Equation~\ref{eq:Kappa_Kappa_prime_swap}.

    \item[] \textbf{Step 3:} Suppose for $X$ and any two knockoff submatrices $D = [\widetilde{\Kappa}_{j_1,1}, \dots, \widetilde{\Kappa}_{j_1,p}]$ and $E = [\widetilde{\Kappa}_{j_2,1}, \dots, \widetilde{\Kappa}_{j_2,p}]$ (where $j_1 \neq j_2$) from the $\ell^{th}$ layer, Equation~\ref{eq:X_Kappa_swap} and Equation~\ref{eq:Kappa_Kappa_prime_swap} hold:
    \begin{align}
    \label{eq:step3_swap}
    [X \quad D]_{\mathrm{swap}(S)} \overset{d}{=} [X \quad D]\quad \text{and}\quad     [D \quad E]_{\mathrm{swap}(S)} \overset{d}{=} [D \quad E]
    \end{align}
    We now want to show that the $(\ell+1)^{th}$ layer, $\widetilde{\Kappa}^{\ell+1}$, satisfies exchangeability. With Equation~\ref{eq:step3_swap}, we only need to prove that for any two knockoff submatrices in $\widetilde{\Kappa}^{\ell+1}$, denoted as $F = [\widetilde{\Kappa}_{j_3,1}, \dots, \widetilde{\Kappa}_{j_3,p}]$ and $G = [\widetilde{\Kappa}_{j_4,1}, \dots, \widetilde{\Kappa}_{j_4,p}]$, we have:
    $[X \quad G]_{\mathrm{swap}(S)} \overset{d}{=} [X \quad G]$, $[F \quad G]_{\mathrm{swap}(S)} \overset{d}{=} [F \quad G]$, $[D \quad G]_{\mathrm{swap}(S)} \overset{d}{=} [D \quad G]$.
    To prove that, note
    \begin{align}
    \label{eq:step3_kappa_ell}
    \widetilde{\Kappa}^\ell \overset{d}{=} \kappa^{\ell+1},\quad\text{then}\quad
    X \overset{d}{=} [\widetilde{\Kappa}_{2^{\ell-1},1}, \ldots, \widetilde{\Kappa}_{2^{\ell-1},p}].
    \end{align}
    so $[F \quad G]_{\mathrm{swap}(S)} \overset{d}{=} [F \quad G]$. Since \( F \) can be any submatrix in \( \kappa^{\ell+1} \), $[\widetilde{\Kappa}_{2^{\ell-1},1}, \ldots, \widetilde{\Kappa}_{2^{\ell-1},p}, G]_{\mathrm{swap}(S)} \overset{d}{=} [\widetilde{\Kappa}_{2^{\ell-1},1}, \ldots, \widetilde{\Kappa}_{2^{\ell-1},p}, G]$. Then by Equation~\ref{eq:step3_kappa_ell} and \emph{A2}, $[X \quad G]_{\mathrm{swap}(S)} \overset{d}{=} [X \quad G]$. And by Equation~\ref{eq:step3_swap} and \emph{A2}, $[D \quad G]_{\mathrm{swap}(S)} \overset{d}{=} [D \quad G]$. Thus Proposition~\ref{prop:exchange} holds for the \((\ell+1)^{th}\) layer.
\end{itemize}
Therefore, the multi-layered knockoffs generated by the recursive process maintain the exchangeability property with the original variables and among themselves.
\end{proof}

\emph{Remark on \emph{Proposition}~\ref{prop:exchange}.} \TheName{} ensures the exchangeability both between each original variable and its knockoff counterparts (in Equation~\ref{eq:X_Kappa_swap}) and between any two knockoff variables rooted on the same original variable (in Equation~\ref{eq:Kappa_Kappa_prime_swap}).

\vspace{-1mm}\subsection{Anomaly-based Significance Test}\vspace{-1mm}
This section outlines the procedure of \emph{anomaly-based significance test} based on the coefficients of Ridgeless regression. Specifically, \TheName{} considers an original variable significant (relative to the response variable) when its coefficient becomes an outlier in the total \(k_{max}+1\) coefficients of the original variable together with its \(k_{max}\) knockoff counterparts. To the end, the proposed method leverages the coefficient differences between the original variables and their knockoff counterparts to estimate the $p$-values, prioritizes variables by their significance, and outputs a ranking list of variables for selection purposes.

\vspace{-2mm}
\begin{algorithm}
\caption{\texttt{estPVal}: Anomaly-based Significance Test}\label{alg:estPVal}
\begin{algorithmic}[1]
\Statex \textbf{Input:} Regression coefficients \(\hat{\boldsymbol{\beta}} = (\hat{\beta}_1, \ldots, \hat{\beta}_p, \hat{\beta}_{1,1}, \ldots, \hat{\beta}_{1,p}, \ldots, \hat{\beta}_{k_{max},1}, \ldots, \hat{\beta}_{k_{max},p})\), where \(\hat{\beta}_i\) represents the coefficient for the \(i^{th}\) original variable and \(\hat{\beta}_{k,i}\) represents the coefficient for the \(k^{th}\) knockoff of the \(i^{th}\) variable.
\Statex \textbf{Output:} $p$-values vector \(\hat{\boldsymbol{P}}\) for original variables \(x_1, \ldots, x_p\).

\For{$i = 1, \ldots, p$}
    \State \(\bar{\beta}_i \gets \sum_{k=1}^{k_{max}} \hat{\beta}_{k,i}/{k_{max}}\) 
    \State \(S_{\hat{\beta}_i} \gets \sqrt{(\sum_{k=1}^{k_{max}} (\hat{\beta}_{k,i} - \bar{\beta}_i)^2) / (k_{max} - 1)}\) \Comment{Compute the sample mean \(\bar{\beta}_i\) and standard deviation \(S_{\hat{\beta}_i}\) of all $k_{max}$ knockoff coefficients for the $i^{th}$ variable}

    \State \(\hat{Z}_i \gets \frac{\bar{\beta}_i - \hat{\beta}_i}{S_{\hat{\beta}_i}}\) \Comment{Compute the Z-statistic by coefficient differences}

    \State $\hat{P}_i \gets 2\Phi(-\left|\hat{Z}_i\right|)$ \Comment{Calculate the two-tailed \(p\)-value using the Z-statistic}
\EndFor

\State \Return \( \hat{\boldsymbol{P}} \) \Comment{The scalar \( \hat{P}_i \) in \( \hat{\boldsymbol{P}} \) is the $p$-value for the \( i^{th} \) variable}.
\end{algorithmic}
\end{algorithm}

\vspace{-1mm}
\subsubsection{Coefficients Distribution and Significance Test}\vspace{-1mm}
Given that the \(k_{max}\) knockoff counterparts of each original variable are independent and identically distributed, we make assumptions as follows.
\begin{itemize}
    \item[\em A3] This work follows existing studies on the distribution of regression coefficients~\cite{Gaussian,clogg1995statistical,paternoster1998using} and assumes that for any original variable, the regression coefficients of its knockoff counterparts are normally distributed. Such that $\hat{\beta}_{k,i} \sim N(\mu_i, \sigma^2_i)$, where $\mu_i$ and $\sigma^2_i$ refer to the mean and variance of the coefficients distribution of knockoff variables for the $i^{th}$ original variable, for each \(k = 1, 2, 3,\ldots, k_{max}\). 

    \item[\em A4] Upon the above assumption, this work further assumes that when the coefficients of each original variable and its knockoff counterparts follows the same distribution, the original variable \(x_i\) is considered to be independent of the response \(\mathbf{y}\).  
    
\end{itemize}
Hereby, we establish the \emph{anomaly-based significance test}, using the sample mean and variance estimated from the coefficients \(\hat{\beta}_{k,i}\) for \(k = 1, 2, 3,\ldots, k_{max}\), as follows.

\begin{definition}[Anomaly-based Significance Test]\label{def:test} For the $i^{th}$ original variable, let denote $\bar{\beta}_i$ and $S^2_{\hat{\beta}_i}$ as the sample mean and variance of the coefficients for its knockoff counterparts calculated as follows
\begin{align}
\bar{\beta}_i = \frac{1}{k_{max}} \sum_{k=1}^{k_{max}} \hat{\beta}_{k,i}, \quad S^2_{\hat{\beta}_i} = \frac{1}{k_{max} - 1} \sum_{k=1}^{k_{max}} (\hat{\beta}_{k,i} - \bar{\beta}_i)^2\ .
\end{align}
The null ($H_0$) and alternative ($H_1$) hypotheses for testing the significance of the $i^{th}$ original variable relevant to the response variable $y$ are given as follows.
\begin{align}
H_0: \hat{\beta}_i \sim N(\bar{\beta}_i, S^2_{\hat{\beta}_i}), \quad H_1: \hat{\beta}_i \not\sim N(\bar{\beta}_i, S^2_{\hat{\beta}_i})\ .
\end{align}
Note that only coefficients of the knockoff variables are considered for distribution modeling, while the coefficient of the original variable is used for testing.
\end{definition}
We have transformed variable selection into a hypothesis testing framework, enabling a systematic and quantifiable evaluation of statistical relevance and supporting informed decisions on variable inclusion.  

\vspace{-1mm}
\subsubsection{Estimation of $p$-values based on Coefficient Differences}\vspace{-1mm}
To implement the above significance test, this step employs a Z-statistic approach~\cite{clogg1995statistical,paternoster1998using} to estimate the $p$-value of the original features. Given the vector of coefficients, \(\hat{\boldsymbol{\beta}}\), as the input, the Z-statistic approach first calculates the normalized difference between coefficients for the $i^{th}$ original variable as follows.
\begin{align}
\hat{Z}_i &:= \frac{\sqrt{k_{max} - 1} (\hat{\beta}_i - \bar{\beta}_i)}{S_{\hat{\beta}_i}},
\end{align}
where \(\hat{\beta}_i\) denotes the estimated coefficient of the \(i^{th}\) original variable, \(\bar{\beta}_i\) and \(S_{\hat{\beta}_i}\) have been defined in \emph{Definition}~\ref{def:test}. Later, we compute the $p$-value of the \(i^{th}\) original variable relevant to the response variable \(\mathbf{y}\) using the equation as follows.
\begin{align}
\hat{P}_i &:= 2\Phi(-|\hat{Z}_i|),
\end{align}
where \(\Phi\) denotes the cumulative distribution function (CDF) of the standard normal distribution. 
A small \(p\)-value suggests a significant, non-random link between the variable and \(\mathbf{y}\), guiding effective variable selection in \TheName{}. 

\vspace{-1mm}\subsubsection{Algorithm Analysis}\vspace{-1mm}
We formulate our main analytical result as the control of False Discovery Rate (FDR) .
\begin{proposition}[FDR Control of Anomaly-based Significance Test]\label{prop:fdr}
Let $V$ be the number of false positives (i.e., original variables falsely selected as significant) and $R$ be the total number of original variables selected as significant by the anomaly-based significance test at a given threshold. The FDR of the test is defined as:
$$
\mathrm{FDR} := \mathbb{E}\left[\frac{V}{R} \mid R>0\right] \cdot \mathbb{P}(R>0)
$$
Under assumptions A3 and A4, the anomaly-based significance test controls the FDR at level $\alpha$ if the $p$-values of the original variables are estimated using the Z-statistic approach and the Benjamini-Hochberg (BH) procedure~\cite{benjamini1995controlling} is applied to select significant variables based on the estimated $p$-values at a target FDR level $\alpha$.
\end{proposition}

\begin{proof}
Under assumptions A3 and A4, we here prove that the test controls the FDR at level $\alpha$ when the $p$-values of the original variables are estimated using the Z-statistic approach and the Benjamini-Hochberg (BH) procedure is applied to select significant variables based on the estimated $p$-values at a target FDR level $\alpha$. 

Let $p$ be the total number of original variables, and $p_0$ be the number of original variables for which the null hypothesis $H_0$ is true (i.e., variables not significantly related to the response). Let $P_1, P_2, \ldots, P_p$ be the true $p$-values of the original variables, and $\hat{P}_1, \hat{P}_2, \ldots, \hat{P}_p$ be the estimated $p$-values using the Z-statistic approach.
\begin{itemize}
    \item[] Step 1: Under assumption A3, for any original variable $i$ for which $H_0$ is true, the Z-statistic $\hat{Z}_i$ follows a standard normal distribution:
$$
\hat{Z}_i \sim N(0, 1) \text{ under } H_0\ .
$$

\item[] Step 2: For any original variable $i$ for which $H_0$ is true, the estimated $p$-value $\hat{P}_i$ using the Z-statistic approach is uniformly distributed on $[0,1]$:
$$
\hat{P}_i \sim U(0, 1) \text{ under } H_0\ .
$$
This follows from the fact that $\hat{P}_i = 2\Phi(-|\hat{Z}_i|)$ and $\hat{Z}_i \sim N(0, 1)$ under $H_0$. 
Furthermore, under assumption A4, the estimated $p$-values $\hat{P}_i$ for variables with true $H_0$ are independent of each other, as the coefficients of the knockoff counterparts for each original variable are assumed to be independent.

\item[] Step 3: The BH procedure, when applied to the estimated $p$-values $\hat{P}_1, \hat{P}_2, \ldots, \hat{P}_p$ at a target FDR level $\alpha$, controls the FDR at level $\alpha$ under the conditions established in steps 1 and 2. This is because the BH procedure controls the FDR at level $\alpha$ when the $p$-values corresponding to the true null hypotheses are uniformly distributed on $[0,1]$ and independent of each other~\cite{benjamini1995controlling}.
\end{itemize}
Therefore, under assumptions A3 and A4, the anomaly-based significance test, which estimates $p$-values using the Z-statistic approach and applies the BH procedure to select significant variables, controls the FDR at level $\alpha$.
\end{proof}

\emph{Remark on Proposition~\ref{prop:fdr}}. The FDR control property ensures that, among the original variables selected as significant by the anomaly-based test, the expected proportion of false positives is at most $\alpha$. This provides a guarantee for the effectiveness of the tests in identifying truly significant variables while controlling the false positive rate.

\vspace{-1mm}
\section{Experiments}\vspace{-1mm}
This section uses both controlled simulations and applied real-world scenarios to evaluate the capabilities of \TheName{}, assessing the effectiveness in selecting variables for enhanced prediction.

\vspace{-2mm}
\begin{algorithm}
\caption{Data Synthesis by Simulation}\label{alg:datasyn}
\begin{algorithmic}[1]
\Statex \textbf{Input:} Sample size \(n\), parameter count \(p\), real parameter count \(p_{\text{real}}\), 
\Statex \phantom{\textbf{Input:}} variable covariance \(\Sigma \in \mathbb{R}^{p \times p}\), error variance \(\sigma^2\).
\Statex \textbf{Output:} Generated data matrix \(X\), training labels \(\mathbf{y}\), coefficient vector \(\beta\).
\State \(X \gets n \times p\) matrix with i.i.d. rows from \(\mathcal{N}(\vec{0}, \Sigma)\)
\State \(\beta_1, \ldots, \beta_{p_{\text{real}}} \gets\) entries \(\overset{\text{i.i.d.}}{\sim} U(0,1)\)
\State \(\boldsymbol{\beta} \gets \mathrm{randomPermute}([\beta_1, \ldots, \beta_{p_{\text{real}}}, 0, \ldots, 0])\)
\Comment{Randomly permute coefficients}
\State \(\boldsymbol{\varepsilon} \gets n\)-dimensional vector with i.i.d. entries from \(\mathcal{N}(0,\sigma^2)\)
\State \(\mathbf{y} \gets X \boldsymbol{\beta} + \boldsymbol{\varepsilon}\)
\State \Return variable matrix \(X\), response vector \(\mathbf{y}\), coefficients \(\boldsymbol{\beta}\).
\end{algorithmic}
\end{algorithm}
\vspace{-3mm}

\vspace{-1mm}\subsection{Simulation with Synthesized Datasets}\vspace{-1mm}
We setup two categories of experiments with the synthesized datasets based on simulations (in Algorithm~\ref{alg:datasyn}) as follows.
\begin{itemize}
    \item  \textbf{Low-dimensional experiments:} These experiments investigate various parameter settings for low-dimensional data generated from simulations. The data matrix $X$ is drawn from a multivariate normal distribution with covariance matrix $\Sigma_{i,j}=\rho^{|i-j|}$, where $\rho = 0.25$, and the error variance is set to $\sigma^2 = 1$. The experiments vary the number of variables $p \in \{80, 100, 150, 180\}$, non-zero components $p_{\text{real}} \in \{10, 20, 30, 40, 50\}$, and sample size $n \in \{85, 100, 120\}$. Each experimental setting is repeated at least 20 times to ensure robust results.  

    \item \textbf{High-dimensional experiment:} The experiment investigates the performance of the proposed method on a dataset with $p=1000$ variables, of which $p_{\text{real}}=30$ are non-zero components in the ground-truth model. The training sample size is set to $n \in \{100, 1000, 3000\}$. Each row of the variable matrix $X$ is generated from a multivariate normal distribution $\mathcal{N}(\vec{0}, \Sigma)$, where the covariance matrix is defined as $\Sigma_{i,j}=\rho^{|i-j|}$ with $\rho = 0.1$. The error variance is set to $\sigma^2 = 0.25$. The experiment is repeated at least 20 times to ensure robust results. 
\end{itemize}
\textbf{Significant variables in ground truth}: Based on the simulation settings in Algorithm~\ref{alg:datasyn}, the variables corresponding to $\beta_1,\ \beta_2,\dots,\ \beta_{P_{real}}$ are \emph{significant variables in ground truth}, as the response variable $\mathbf{y}$ is generated as the noisy linear combination of these variables with non-zero coefficients (lines 2--5 in Algorithm~\ref{alg:datasyn}).

In our experiments, we employ various regression and variable selection methods, including Ridge, Lasso, ElasticNet, Model-X and Multiple Knockoff filters~\cite{modelXKnockoff,nguyen2020aggregation}, for evaluation. For simplicity, as the generated data are already normalized, we directly compare the regression coefficients of Lasso, ElasticNet, and Ridge as the importance of variables. All these implementations are derived from \texttt{scikit-learn}, \texttt{knockpy} and \texttt{multiknockoffs}, while hyper-parameters are tuned best through cross-validation. Note that \TheName{} utilizes $\ell_{max}=3$ layers of hierarchical knockoffs (consisting of a total of 7 sets of knockoff matrices), while Multiple Knockoff employs 25 knockoff matrices according to the package default settings. To evaluate the effectiveness for variable selection, we use Area Under Curve (AUC) of the Receiver Operating Characteristic (ROC) curve to report our experiment results. As shown in Algorithm~\ref{alg:estPVal}, the raw output of \TheName{} consists of a ranking list of variables ordered by their p-values. The AUC measures the true positive rate (TPR) and false positive rate (FPR) as variables are sequentially selected from the ranking list.

\begin{table}[ht]
  \centering
  \vspace{-10mm}
  \caption{Results (AUC) for Simulation-based Evaluation (values at the \textbf{first places} and the \underline{second places}). Settings 1: $p = 80$, $p_{real} = 10$, $n = 100$, 2:  $p = 100$, $p_{real} = 10$, $n = 100$, 3: $p = 150$, $p_{real} = 10$, $n = 100$, 4: $p = 180$, $p_{real} = 10$, $n = 100$, 5: $p = 100$, $p_{real} = 20$, $n = 100$, 6: $p = 100$, $p_{real} = 30$, $n = 100$, 7: $p = 100$, $p_{real} = 40$, $n = 100$, 8: $p = 100$, $p_{real} = 50$, $n = 100$, 9:  $p = 100$, $p_{real} = 10$, $n = 85$, 10: $p = 100$, $p_{real} = 10$, $n = 120$, 11: $p = 1000$, $p_{real} = 30$, $n = 3000$, 12: $p = 1000$, $p_{real} = 30$, $n = 100$ and 13: $p = 1000$, $p_{real} = 30$, $n = 1000$. }\label{tbl:auc_value}
  \begin{tabular}{@{}ccccccc@{}}
    \toprule
    Settings & \textbf{\TheName{}} & \textbf{Model-X K.} & \textbf{Multiple K.} & \textbf{Ridge}  & \textbf{Lasso} & \textbf{ElasticNet} \\
    \hline
    \multicolumn{7}{c}{The Low-Dimensional Experiments} \\   \midrule
    1 & $\mathbf{0.788\pm 0.068}$ & $\underline{0.783 \pm 0.075}$ & $0.733 \pm 0.062$ & $0.700 \pm 0.106$ & $0.520 \pm 0.000$ & $0.616 \pm$ 0.076 \\
    2 & $\mathbf{0.784 \pm 0.063}$ & $\underline{0.776 \pm 0.071}$ & $0.733 \pm 0.063$ & $0.603 \pm 0.090$ & $0.282 \pm 0.000$ & $0.388 \pm 0.083$\\
    3 & $\mathbf{0.796 \pm 0.071}$ & $\underline{0.777 \pm 0.082}$ & $0.732 \pm 0.065$ & $0.740 \pm 0.076$ & $0.663 \pm 0.000$ & $0.708 \pm 0.041$ \\
    4 & $\mathbf{0.792 \pm 0.065}$ & $\underline{0.778 \pm 0.067}$ & $0.741 \pm 0.059$ & $0.764 \pm 0.059$ & $0.522 \pm 0.003$ & $0.603 \pm 0.051$ \\
    5 & $\underline{0.696 \pm 0.060}$ & $0.688 \pm 0.057$ & $\mathbf{0.734 \pm 0.060}$ & $0.552 \pm 0.075$ & $0.456\pm 0.000$ & $0.477\pm 0.025$ \\
    6 & $\underline{0.663 \pm 0.045}$ & $0.650 \pm 0.059$ & $\mathbf{0.738 \pm 0.062}$ & $0.534 \pm 0.055$ & $0.408 \pm 0.003$ & $0.431\pm 0.027$ \\
    7 & $\underline{0.646 \pm 0.051}$ & $0.638 \pm 0.050$ & $\mathbf{0.748 \pm 0.060}$ & $0.526 \pm 0.059$ & $0.595 \pm 0.000$ & $0.607 \pm 0.013$ \\
    8 & $\underline{0.621 \pm 0.051}$ & $0.617 \pm 0.052$ & $\mathbf{0.744 \pm 0.611}$ & $0.529 \pm 0.050$ & $0.530\pm 0.000$ & $0.533\pm 0.012$ \\
    9 & $\mathbf{0.765\pm 0.079}$ & $\underline{0.752\pm 0.073}$ & $0.745 \pm 0.056$ & $0.648\pm 0.092$ & $0.282\pm 0.000$ & $0.398\pm 0.085$ \\
    10 & $\mathbf{0.803\pm 0.067}$ & $\underline{0.791\pm 0.063}$ & $0.742 \pm 0.059$ & $0.700 \pm 0.091$ & $0.282 \pm 0.000$& $0.384 \pm 0.095$ \\ 
   \midrule
    \multicolumn{7}{c}{The High-Dimensional Experiment} \\   \midrule
    11   & $\mathbf{0.985\pm 0.011}$ & $\underline{0.981 \pm 0.013}$ & $0.737 \pm 0.056$ & $0.969 \pm 0.016$ & $0.366 \pm 0.000$ & $0.366 \pm 0.000$ \\
    12   & $0.677 \pm 0.003$ & $0.675 \pm 0.003$ & $\mathbf{0.746 \pm 0.047}$ & $\underline{0.719 \pm 0.002}$ & $0.507 \pm 0.007$ & $0.532 \pm 0.007$\\
    13   & $\underline{0.894 \pm 0.001}$ & $\mathbf{0.902 \pm 0.001}$ & $0.736 \pm 0.059$ & $0.721 \pm 0.004$ & $0.421 \pm 0.000$ & $0.421 \pm 0.000$\\
    \bottomrule
  \end{tabular}
\vspace{-3mm}
\end{table}

\vspace{-1mm}
\subsubsection{Experiment Results}\vspace{-1mm}
The experiment results for variable selection, as presented in Table~\ref{tbl:auc_value}, showcase the performance of different methods in both low-dimensional and high-dimensional settings. Across all 13 experimental settings, \TheName{} achieves the highest AUC values in 7 of them and takes the second places in 5 of them, indicating its superior ability to prioritize significant variables. Especially, for the low-dimensional experiments (settings 1-10), \TheName{} always achieves the highest or the second highest AUC values. Similarly, in the high-dimensional experiments (settings 11-13), \TheName{} maintains its superior performance, particularly when the sample size is sufficient (settings 11 and 13). These results highlight the ability of \TheName{} to accurately identify and rank significant variables, across different numbers of variables ($p$), non-zero components ($p_{\text{real}}$), and sample sizes ($n$).

\vspace{-1mm}
\subsubsection{Discussions}\vspace{-1mm}
\TheName{} and Knockoff methods generally outperform Lasso or ElasticNet in variable selection, achieving higher AUC, due to their specific design to control FDRs while maintaining power via so-called ``selective inference''~\cite{highDimFixKnock}. Utilizing coefficients from Lasso and ElasticNet for variable selection can even result in an AUC below 0.5, indicative of performance inferior to random guessing. Methods like Lasso or ElasticNet can still fit and predict data well, even when their non-zero coefficients are not assigned to the true variables, due to the inconsistency in coefficient estimates~\cite{lee2022lasso}. Specifically, in high-dimensional scenarios where an extremely sparse model is fitted, almost all the coefficients are estimated as $0$. Consequently, all variables with zero coefficients may be considered to be equally irrelevant, making it impracticable to distinguish or select meaningful variables based solely on their coefficients.

\begin{table}[ht]
  \centering
  \vspace{-5mm}
  \caption{Overview of Datasets Used for Model Evaluation}\label{tab:datasets}
  \begin{tabular}{@{}lccc@{}}
    \toprule
    \textbf{Dataset} & \textbf{\#Variables} & \textbf{\#Train/Test} & \textbf{Prediction Task} \\
    \midrule
    Alon Dataset~\cite{Alon} & 2000 & 37/25 & Classification (genetics) \\
    Communities and Crime ~\cite{Crime}& 100 & 21/9 & Regression (sociology) \\
    Superconductivity~\cite{superconductivity} & 81 & 360/40 & Regression (chemistry) \\
    Appliances Energy Prediction~\cite{energy} & 27 & 225/25 & Regression (environment)\\
    \bottomrule
  \end{tabular}
  \vspace{-5mm}
\end{table}

\subsection{Evaluation with Realistic Datasets}
Table~\ref{tab:datasets} provides the profile of every dataset for evaluation, including the number of variables, the number of samples, and the context of tasks. To carry out the experiments, we use following two categories of baseline algorithms for comparisons.
\begin{itemize}
    \item Statistical methods such as model-X Knockoff, multiple Knockoff, Ridge regression, Lasso, and ElasticNet are utilized to select variables by their statistical significance. For every comparison here, \TheName{} and these methods are set to selected a fixed-length of variables in every experiment (please refer to Section~\ref{sec:var_select}). 

    \item Global optimization-based feature selectors, including Harris Hawk Optimization (HHO)~\cite{hho}, Jaya Algorithm~\cite{jaya}, Sine Cosine Algorithm (SCA)~\cite{sca}, Salp Swarm Algorithm (SSA)~\cite{ssa}, and Whale Optimization Algorithm (WOA)~\cite{woa} are used to search the subset of features. In the experiment, the number of variables to be selected by \TheName{} is determined via cross-validation for fair comparison, as all global optimization-based methods leverage the cross-validated performance measure as the search objective for feature selection (please refer to Section~\ref{sec:var_select}). 
\end{itemize}
Note that \TheName{} employs a hierarchical knockoff structure spanning $\ell_{max}=4$ layers (totaling 15 sets of knockoff matrices). For classification tasks, we use misclassification rate for evaluation and comparison. For regression tasks, we report mean squared error (MSE) as the measurements of performance.

\begin{table}[ht]
\vspace{-10mm}
  \centering
  \caption{Result comparisons with statistical methods for fixed-length selection (values at the \textbf{first places} and the \underline{second places})}\label{tbl:stat_selector}
  \begin{tabular}{@{}cccccccc@{}}
    \toprule
    \#Selected Variables & \TheName{} & \textbf{Model-X K.} & \textbf{Multiple K.} & \textbf{Ridge} & \textbf{Lasso} & \textbf{ElasticNet} & \textbf{Lars} \\
    \midrule
    \multicolumn{8}{c}{Misclassification Rate on Alon Dataset}\\ \midrule
    1 & $\mathbf{0.400}$ & $\underline{0.600}$ & $\mathbf{0.400}$ & $\mathbf{0.400}$ & $\mathbf{0.400}$ & $\mathbf{0.400}$ & $\mathbf{0.400}$ \\
    2 & $\mathbf{0.120}$ & 0.320 & $\underline{0.160}$ & 0.400 & 0.440 & 0.400 & 0.360 \\
    4 & $\mathbf{0.160}$ & $\underline{0.240}$ & 0.320 & 0.280 & 0.280 & 0.400 & 0.360 \\
    8 & \underline{0.240} & 0.280 & 0.360 &\underline{0.240} & $\mathbf{0.160}$ & \underline{0.240} & 0.360 \\
    \midrule
    \multicolumn{8}{c}{MSE on Communities and Crime Dataset ($\times10^{-7}$)}\\ \midrule     
    1 & $\mathbf{3.147}$ & $\underline{24.200}$ & $\underline{24.200}$ &$2084$ & $\underline{24.200}$ & $\underline{24.200}$ & $5869$ \\
    2 & $\mathbf{5.447}$ & $\underline{24.440}$ & $272.070$ &$2187$ & $\underline{24.440}$ & $\underline{24.440}$ & $5120$ \\
    4 & $\mathbf{63.204}$ & $17630$ & $8745$ &$\underline{2522}$ & $17630$ & $17630$ & $196100$ \\
    8 & $\mathbf{401.5}$ & $88530$ & $\underline{8799}$ &$9213$ & $88530$ & $88530$ & $182800$ \\  
    \midrule
    \multicolumn{8}{c}{MSE on Superconductivity Dataset ($\times10^{2}$)}\\ \midrule
    1 & $\mathbf{5.003}$ & $9.023$ & $\underline{6.267}$ & $9.730$ & $9.847$ & $9.847$ & $7.034$ \\
    4 & $\mathbf{2.714}$ & $9.874$ & $\underline{4.358}$ & $11.125 $ & $6.034$ & $6.462$ & $5.049$ \\
    8 & $3.047$ & $8.826$ & $\underline{2.436}$ & $4.850$ & $3.095$ & $\mathbf{2.348}$ & $3.064$ \\
    \midrule
    \multicolumn{8}{c}{MSE on Appliances Energy Prediction Dataset ($\times10^{3}$)}\\ \midrule

    1 & $\mathbf{3.861}$ & $\underline{5.527}$ & $\mathbf{3.861}$ & $\underline{5.527}$ & $\underline{5.527}$ & $5.321$ & $5.624$ \\
    2 & $\mathbf{3.855}$ & $5.357$ & $\underline{4.295}$ & $5.079$ & $5.357$ & $5.214$ & $5.160$ \\
    4 & $\mathbf{3.720}$ & $4.922$ & $\underline{4.444}$ & $4.860$ & $5.066$ & $5.834$ & $5.059$ \\
    8 & $\mathbf{3.895}$ & $6.581$ & $4.504$ & $5.109$ & $6.581$ & $\underline{4.243}$ & $4.793$ \\  
    \bottomrule
  \end{tabular}
\end{table}

\begin{table}
\vspace{-3mm}
  \centering
  \caption{Comparisons with global optimization-based methods for varying-length selection (values at the \textbf{first places} and the \underline{second places}).}\label{tbl:feature_selector}
  \begin{tabular}{@{}ccccccc@{}}
    \toprule
    & \TheName{} & \textbf{HHO} & \textbf{Jaya} & \textbf{SCA} & \textbf{SSA} & \textbf{WOA} \\
    \midrule
    \multicolumn{7}{c}{Misclassification Rate on Alon Dataset}\\ \midrule
    Log. Reg. & $\mathbf{0.160}$ & 0.400 & \underline{0.200} & 0.320 & 0.240 & 0.240 \\
    Rand. Forest & $\mathbf{0.160}$ & 0.320 & \underline{0.200} & 0.280 & $\mathbf{0.160}$ & 0.400 \\
    GBDT & 0.320 & 0.480 & \underline{0.240} & 0.280 & $\mathbf{0.160}$ & 0.320 \\
    SVM-Linear & $\mathbf{0.200}$ & 0.440 & $\mathbf{0.200}$ & $\mathbf{0.200}$ & $\underline{0.280}$ & 0.360 \\
    SVM-Radial & $\mathbf{0.200}$ & \underline{0.240} & 0.320 & 0.400 & 0.320 & 0.560 \\
    SGDClassifer & 0.240 & 0.240 & \underline{0.200} & $\mathbf{0.160}$ & \underline{0.200} & \underline{0.200} \\ 
    \midrule
    \multicolumn{7}{c}{MSE on Communities and Crime Dataset ($\times10^{-3}$)}\\ \midrule
    Lin. Reg. & $\mathbf{3.147 \times 10^{-4}}$ & $6.473 \times 10^{-2}$ & $9.513$ & $\underline{6.342 \times 10^{-3}}$ & $8.156$ & $0.130$\\
    Rand. Forest & $\mathbf{1.733 \times 10^{-30}}$ & $0.811$ & $4.887$ & $0.943$ & $\underline{0.182}$ & $3.591$ \\
    GBDT & $\mathbf{5.420 \times 10^{-13}}$ & $\underline{4.004 \times 10^{-12}}$ & $0.769$ & $0.256$ & $0.120$ & $13.98$ \\
    SVM-Linear & $1.947$ & $1.486$ & $2.263$ & $\mathbf{0.574}$ & $\underline{1.194}$ & $4.618$ \\
    SVM-Radial & $3.700$ & $\underline{1.169}$ & $3.327$ & $1.683$ & $3.170$ & $\mathbf{0.740}$ \\
    Neural Net (2) & $4.808 \times 10^{-2}$ & $\underline{5.676 \times 10^{-3}}$ & $9.548 \times 10^{-2}$ & $\mathbf{5.656 \times 10^{-3}}$ & $5.893 \times 10^{-3}$ & $6.363 \times 10^{-3}$ \\
    Neural Net (3) & $5.676 \times 10^{-3}$ & $\underline{5.357 \times 10^{-3}}$ & $2.601 \times 10^{-2}$ & $5.739 \times 10^{-3}$ & $\mathbf{3.595 \times 10^{-3}}$ & $5.698 \times 10^{-3}$ \\
    \midrule
    \multicolumn{7}{c}{MSE on Superconductivity Dataset ($\times10^{1}$)}\\ \midrule
    Lin. Reg. & $\mathbf{13.594}$ & $41.925$ & $\underline{18.036}$ & $67.446$ & $58.543$ & $53.824$ \\
    Rand. Forest & $13.248$ & $7.429$ & $11.125$ & $\underline{7.418}$ & $\mathbf{7.070}$ & $8.032$ \\
    GBDT & $\underline{8.452}$ & $8.596$ & $\mathbf{7.712}$ & $12.410$ & $9.211$ & $9.030$ \\
    SVM-Linear & $\mathbf{31.634}$ & $65.979$ & $468.9$ & $\underline{33.935}$ & $71.169$ & $71.979$ \\
    SVM-Radial & $\mathbf{36.210}$ & $\underline{62.506}$ & $70.990$ & $73.750$ & $74.102$ & $68.142$ \\
    Neural Net (2) & $62.685$ & $78.623$ & $215.5$ & $\mathbf{57.320}$ & $61.703$ & $\underline{60.868}$ \\
    Neural Net (3) & $\mathbf{51.826}$ & $61.643$ & $61.142$ & $\underline{54.958}$ & $81.451$ & $59.789$ \\
    \midrule
    \multicolumn{7}{c}{MSE on Appliances Energy Prediction Dataset ($\times10^{3}$)}\\ \midrule
    Lin. Reg. & $10.303$ & $\underline{2.331}$ & $\mathbf{2.310}$ & $4.904$ & $4.242$ & $4.590$ \\
    Rand. Forest  & $11.103$ & $8.608$ & $\mathbf{3.939}$ & $6.577$ & $\underline{4.904}$ & $5.392$ \\
    SVM-Linear & $\underline{4.480}$ & $4.909$ & $4.855$ & $5.091$ & $\mathbf{4.479}$ & $4.957$ \\
    SVM-Radial & $\mathbf{4.600}$ & $5.162$ & $4.966$ & $5.166$ & $5.160$ & $\underline{4.910}$ \\
    Neural Net (2) & $\underline{5.202}$ & $5.340$ & $\mathbf{5.191}$ & $5.291$ & $5.224$ & $5.216$ \\
    Neural Net (3) & $\underline{5.231}$ & $\mathbf{5.216}$ & $5.258$ & $5.347$ & $5.291$ & $8.644$ \\
    \bottomrule
  \end{tabular}
\end{table}

\subsubsection{Experiment Results}
In Table~\ref{tbl:stat_selector}, we present the comparison results between \TheName{} and statistical variable selection methods under the settings of fixed-length selection (1--10 selected variables). Due to the page length, we only present 1, 2, 4, and 8 here. Specifically, we assess the misclassification rates of chosen variables through logistic regression model, alongside analyzing the MSE of these variables when employing linear regression for prediction.
The results demonstrate that \TheName{} consistently outperforms the other methods in terms of misclassification rate and MSE when selecting a small number of variables (1, 2, or 4). This superior performance is particularly evident in the Communities and Crime dataset, where \TheName{} achieves significantly lower MSE values compared to the other methods. However, as the number of selected variables increases to 8, the performance of \TheName{} becomes comparable to or slightly worse than some of the other methods, such as Lasso and ElasticNet, in certain datasets like the Alon and Superconductivity datasets. Note that the misclassification rate on Alon dataset could be even higher than 0.5, as the samples are extremely class-imbalanced.

Table~\ref{tbl:feature_selector} compares the performance of \TheName{} against several global optimization-based feature selectors for the four datasets, under varying-length selection settings. The results demonstrate that \TheName{} not only achieves the highest accuracy (measured by misclassification rate or MSE) when selecting features for a specific machine learning algorithm on the majority of datasets but also almost delivers the highest accuracy across all machine learning models for each dataset. We are encouraged to see \TheName{} outperforms the global optimization-based approaches for feature selection, as these methods leveraged cross-validated accuracy as objectives to directly search the subset of features, for predictive modeling.

\subsubsection{Discussions}
Comparing to sparse linear models and global optimization-based feature selectors, the core advantage of \TheName{} (as well as the Knockoff baselines) is its capability in testing the conditional independence between variables and the response for prediction purposes~\cite{watson2021testing}. Apparently, filtering the variables that are conditional independent with the responses is more effective than just estimating the sparsest set of coefficients, resulting in more generalizable models and preventing models learning from noise as signal. Furthermore, such test helps in distinguishing  causation from correlation. This distinction is crucial in scenarios where understanding causal relationships is as important as achieving high prediction accuracy~\cite{yu2020causality}. Finally, \TheName{} outperforms existing Knockoff methods by generating multiple sets of knockoff variables and integrating them into a Ridgeless regression to enhance capacity. It evaluates the significance of each original variable by comparing its coefficient against the distribution of its knockoff counterparts, utilizing a larger reference group of null variables, offering a more precise assessment of impact of original variables in prediction.  

Of-course, our experiments are with some limitations. For example, we have not compare \TheName{} with some computation-intensive neural feature selectors such as sparse autoencoder~\cite{atashgahi2022quick}. Moreover, though \TheName{} provides an estimate of $p$-value, we have not evaluate the feature selection procedure by thresholding $p$-values (such as $\leq 0.05$). Actually, we have done some trials by setting the threshold to $\leq 0.05$ for the real datasets. \TheName{} can deliver a test set MSE of $4.015 \times 10^{-5}$ with the Communities and Crime dataset, and $3.047 \times 10^{2}$ with the Superconductivity dataset, even without cross-validation to determine the number of selected variables.

\section{Discussion and Conclusions}
This work introduces \TheName{} to enhance variable selection, which addresses the model capacity issue in regression models used in the common Model-X Knockoff filter. The proposed approach first generates multiple sets of knockoff variables through \emph{recursive generation}, and then integrates every original variable and its multiple knockoff counterparts into a \emph{Ridgeless regression}. \TheName{} refines coefficient estimates with better capacity, and employs an \emph{anomaly-based significance test} for robust variable selection. \TheName{} can either pickup a predefined number of variables or optimizes the number of selected variables by cross-validation. Extensive experiments demonstrate superior performance compared to existing methods in relevant variable discovery against ground truth variables controlled by simulations and supervised feature selection for classification/regression tasks.

\section*{Declaration}\vspace{-2mm}
\begin{itemize}
\item Funding -  Not applicable

\item Conflicts of interest/Competing interests - Not applicable

\item Ethics approval - No data have been fabricated or manipulated to support your conclusions. No data, text, or theories by others are presented as if they were our own. Data we used, the data processing and inference phases do not contain any user personal information. This work does not have the potential to be used for policing or the military.

\item Consent to participate - Not applicable

\item Consent for publication - Not applicable

\item Availability of data and material - Experiments are based on publicly available open-source datasets.

\item Code availability - All codes will be released after acceptance.

\item Authors' contributions - H.X contributed the original idea. X.Z conducted experiments. X.Z and H.X wrote the manuscript. Y.C involved in discussion and wrote part of the manuscript. X.Z and H.X share equal technical contribution.
\end{itemize} 

\scriptsize{ 
\bibliography{main}}

\end{document}